\documentclass[conference]{IEEEtran}

\usepackage{cite}
\usepackage{amsmath,amssymb,amsfonts}
\usepackage{algorithm}
\usepackage{algorithmic}
\usepackage{graphicx}
\usepackage{textcomp}
\usepackage{xcolor}

\newtheorem{theorem}{\bf Theorem}

\newtheorem{proof}{Proof}

\begin{document}

\title{Heterogeneous Learning Rate Scheduling for Neural Architecture Search on Long-Tailed Datasets}

\author{\IEEEauthorblockN{Chenxia Tang}
\IEEEauthorblockA{\textit{Department of Computer Science} \\
\textit{University of Technology and Science of China}\\
 \\
tomorrowdawn@mail.ustc.edu.cn}
}

\maketitle

\begin{abstract}
In this paper, we attempt to address the challenge of applying Neural Architecture Search (NAS) algorithms, specifically the Differentiable Architecture Search (DARTS), to long-tailed datasets where class distribution is highly imbalanced. We observe that traditional re-sampling and re-weighting techniques, which are effective in standard classification tasks, lead to performance degradation when combined with DARTS. To mitigate this, we propose a novel adaptive learning rate scheduling strategy tailored for the architecture parameters of DARTS when integrated with the Bilateral Branch Network (BBN) for handling imbalanced datasets. Our approach dynamically adjusts the learning rate of the architecture parameters based on the training epoch, preventing the disruption of well-trained representations in the later stages of training. Additionally, we explore the impact of branch mixing factors on the algorithm's performance. Through extensive experiments on the CIFAR-10 dataset with an artificially induced long-tailed distribution, we demonstrate that our method achieves comparable accuracy to using DARTS alone. And the experiment results suggest that re-sampling methods inherently harm the performance of the DARTS algorithm. Our findings highlight the importance of careful data augment when applying DNAS to imbalanced learning scenarios.
\end{abstract}


\section{Introduction}

In recent years, with the continuous development of deep learning, factors limiting model performance have received increasing attention. Among these factors, model architecture plays a significant role. Most models rely on expert-designed structures, leading to lengthy experimental cycles and a heavy dependence on experience and grid search. To automate this process, numerous neural architecture search (NAS) techniques have been proposed\cite{zoph2016nas, liu2018darts, xie2018snas, liang2019darts+, chen2021progressive}, roughly categorized as reinforcement learning-based NAS, evolutionary algorithm-based NAS, and gradient-based NAS. Among these techniques, DARTS (Differentiable Architecture Search)\cite{liu2018darts} has stood out due to its extremely low training cost and comparable performance. During the training of DARTS models, balanced datasets are typically used, where the number of samples for each image class is similar. However, when applying DARTS to real-world production environments, we often encounter long-tailed distributions, where a small number of tail classes have very limited data, resulting in an extremely low sample count. We have observed that imbalanced datasets can severely impair the performance of DARTS, shown in Table \ref{tab:prelim-exp}, making this issue a pressing concern.

Several well-established solutions have been proposed to address the training problem on long-tailed distributions. These techniques can be broadly categorized into three classes:

\begin{itemize}
    \item \textbf{Re-sampling}: This involves modifying the distribution of the training data to make it more balanced. The two main approaches are oversampling the minority class, such as through duplicating or synthesizing new examples \cite{chawla2002smote}, and undersampling the majority class to reduce its prevalence. Re-sampling techniques aim to expose the model to a more balanced distribution during training.
    
    \item \textbf{Re-weighting}: Instead of altering the training data, re-weighting methods address imbalance by modifying the loss function. The idea is to assign higher weights to examples from the minority class, effectively making their contribution to the loss more significant \cite{lin2017focal}. This encourages the model to focus more on learning the minority class. Cost-sensitive learning is a popular re-weighting approach.
    
    \item \textbf{Mix-up}: Mix-up is a data augmentation technique that involves creating new training examples by linearly interpolating between existing ones \cite{zhang2017mixup}. For imbalanced learning, mix-up can be used to generate new examples of the minority class by interpolating between minority examples and majority examples. This can help expand the minority class and make its distribution more diverse.
\end{itemize}

These techniques can be used independently or in combination to tackle imbalanced learning problems. The choice of method depends on factors such as the severity of the imbalance, the size of the dataset, and the specific domain and task. Properly addressing imbalance is crucial for developing models that perform well across all classes.

Cumulative learning\cite{zhou2020bbn} is a special case of re-sampling, where two branch networks are combined, one of which has its input transformed through re-sampling. Networks like Bilateral-Branch Network(BBN)\cite{zhou2020bbn} have achieved surprisingly good results using this approach. However, our experiments indicate that simply combining DARTS with this strategy is insufficient to improve performance and can even lead to a significant decline. To address this, we propose a series of improvements based on our observation that gradually adjusting the learning rate of architecture parameters according to the epoch can prevent the re-balanced data from disrupting the well-trained representations in the later stages of the experiment. Furthermore, we have noticed that the mixing factor in cumulative learning is usually set to 0.5 during the inference process, but this is not the case in the DARTS algorithm. This suggests that the final network performance largely benefits from the classification head, and treating the two branches as equal is not feasible.

In this paper, we thoroughly investigate the behavior of the DARTS algorithm under the cumulative learning strategy. Specifically, we employ adaptive learning rates for different architecture parameters and explore the impact of branch mixing factors on the algorithm's results. We point out that adjusting the branch mixing factor and learning rate for different branches' classification heads is equivalent, but it has no effect on the shared weights, i.e., the backbone part. This leads to the data from the re-balanced branch disrupting the well-trained backbone in the later stages of training. Since the DARTS algorithm requires simultaneous optimization of weights and architecture, this introduces significant instability, undermining the training process. Furthermore, in the experimental section, we also discuss the influence of other similar strategies, such as freezing the backbone or changing mixing factor evolution strategies. The experimental results strongly demonstrate that, compared to directly using cumulative learning or re-sampling, we achieve substantial improvements. However, the final effect is only slightly better than using DARTS directly. We conclude that re-sampling methods inherently harm the performance of the DARTS algorithm.

The paper is structured as follows: Section 2 provides a comprehensive review of related works, focusing on Neural Architecture Search (NAS) and long-tailed learning. Section 3 details our proposed adaptive learning rate scheduling strategy, explaining its rationale and implementation. Section 4 outlines the experimental setup, including the datasets used, the methodology for evaluating the proposed strategies, and preliminary results. This section also includes a discussion on the implications of our findings and comparisons with existing methods. Finally, Section 5 concludes the paper, summarizing the key contributions and suggesting future research directions.

\section{Related Work}

\subsection{Differentiable Neural Architecture Search} Neural architecture search (NAS) aims to automate the design of neural network architectures. Early NAS methods relied on discrete search spaces and non-differentiable optimization techniques such as reinforcement learning \cite{zoph2016nas} and evolutionary algorithms \cite{krizhevsky2009learning}. However, these approaches are computationally expensive and require a large number of architecture evaluations. To address these limitations, differentiable NAS (DNAS) methods have been proposed, which enable the search process to be optimized using gradient-based techniques \cite{liu2018darts}. To the best of our knowledge, the DARTS algorithm still lacks theoretical guarantees, but numerous studies have verified its effectiveness with strong experimental results\cite{liang2019darts+, xie2018snas}. Compared to traditional RL or evolutionary methods, DARTS has a significant advantage in training cost. Most RL or evolutionary methods require thousands of GPU days\cite{zoph2016nas}, while DARTS only needs $<2$ days. 

One prominent DNAS approach is DARTS (Differentiable Architecture Search) \cite{liu2018darts}, which introduces a continuous relaxation of the discrete architecture space, allowing for efficient search using gradient descent. DARTS has been widely adopted and has inspired many subsequent DNAS methods \cite{xu2019pc, chen2021progressive}. PC-DARTS \cite{xu2019pc} improves upon DARTS by introducing partial channel connections, reducing the memory footprint and enabling the search of larger architectures. Progressive DARTS \cite{chen2021progressive} addresses the instability issues of DARTS by progressively increasing the depth of the searched architectures during the optimization process.

Other notable DNAS methods include NAO (Neural Architecture Optimization) \cite{luo2018neural}, which uses an auto-encoder to map architectures to a continuous space, and SNAS (Stochastic Neural Architecture Search) \cite{xie2018snas}, which introduces a stochastic sampling strategy to improve the search efficiency. More recently, DARTS+ \cite{liang2019darts+} and GDAS (Gradient-based search using Differentiable Architecture Sampling) \cite{dong2019searching} have been proposed to further enhance the stability and efficiency of DNAS.

Despite the advancements in DNAS, challenges such as the gap between search and evaluation performance, the computational cost of search, and the interpretability of the searched architectures remain active areas of research \cite{zela2019understanding, li2020random}.

\subsection{Vision Classification} Image classification is a fundamental task in computer vision, aiming to assign predefined labels to input images. Deep convolutional neural networks (CNNs) have achieved remarkable success in this task, with architectures such as AlexNet \cite{krizhevsky2012imagenet}, VGGNet \cite{simonyan2014very}, and ResNet \cite{he2016deep} setting new benchmarks on large-scale datasets like ImageNet \cite{deng2009imagenet}.

Recent advancements in vision classification have focused on designing more efficient and powerful architectures. MobileNets \cite{howard2017mobilenets, sandler2018mobilenetv2} introduce depthwise separable convolutions to reduce the computational cost while maintaining high accuracy. EfficientNets \cite{tan2019efficientnet} propose a systematic approach to scale up CNNs in terms of depth, width, and resolution, achieving state-of-the-art performance with improved efficiency. Vision Transformers (ViT) \cite{dosovitskiy2020image} adapt the self-attention mechanism from natural language processing to vision tasks, showing promising results on image classification benchmarks.

Other notable contributions to vision classification include the use of attention mechanisms \cite{hu2018squeeze, woo2018cbam, zhang2022resnest}, neural architecture search \cite{zoph2018learning, liu2018darts}, and knowledge distillation \cite{hinton2015distilling, tian2019contrastive}. Attention mechanisms help the model focus on relevant features, while NAS automates the design of architectures tailored for specific tasks. Knowledge distillation transfers knowledge from larger, more complex models to smaller, more efficient ones, enabling deployment on resource-constrained devices.

\subsection{Imbalanced Learning} Imbalanced learning deals with classification tasks where the number of samples per class is not evenly distributed. This is a common challenge in real-world applications, such as medical diagnosis, fraud detection, and object detection, where some classes may have significantly fewer samples than others. Conventional machine learning algorithms often struggle in imbalanced scenarios, as they tend to be biased towards the majority class \cite{branco2016survey}.

To address this issue, various techniques have been proposed. Resampling methods, such as oversampling the minority class (e.g., SMOTE \cite{chawla2002smote}) or undersampling the majority class \cite{liu2008exploratory}, aim to balance the class distribution. Cost-sensitive learning assigns higher misclassification costs to the minority class, encouraging the model to focus on these samples. Ensemble methods, like SMOTEBoost \cite{chawla2003smoteboost} and RUSBoost \cite{seiffert2009rusboost}, combine resampling techniques with boosting algorithms to improve performance on imbalanced datasets.

In the context of deep learning, class-balanced loss functions \cite{cui2019class} and focal loss \cite{lin2017focal} have been proposed to address imbalanced learning. Class-balanced loss functions re-weight the loss based on the inverse class frequency, while focal loss down-weights the contribution of easy examples and focuses on hard, misclassified samples. Other approaches include meta-learning \cite{ren2018learning} and transfer learning \cite{wang2017balanced} to leverage information from related tasks or domains.

Recent research has also explored the use of generative models, such as GANs \cite{mariani2018bagan} and VAEs \cite{huang2022ada}, to synthesize additional samples for the minority class. These generated samples can help alleviate the class imbalance and improve the model's generalization ability.

Despite the progress made in imbalanced learning, it remains an ongoing challenge, particularly in the presence of extreme imbalances, noisy labels, and complex data distributions \cite{branco2016survey, johnson2019survey}. Further research is needed to develop more robust and adaptive methods for handling imbalanced datasets in various domains.

%%%这里需要介绍darts算法的大概过程.

\section{Preliminary}

\subsection{CNN}

Convolutional Neural Networks (CNNs) are a class of deep learning models designed to process grid-like data, such as images. CNNs have achieved state-of-the-art performance in various computer vision tasks, including image classification, object detection, and semantic segmentation.

\subsubsection{Convolutional Layers}

The core building block of a CNN is the convolutional layer. Given an input tensor $\mathbf{X} \in \mathbb{R}^{H \times W \times D}$, where $H$, $W$, and $D$ represent the height, width, and depth (number of channels) of the input, respectively, a convolutional layer applies a set of learnable filters $\mathbf{W} \in \mathbb{R}^{K \times K \times D \times F}$ to the input. Here, $K$ represents the spatial dimensions of the filters, and $F$ is the number of filters. The output of a convolutional layer is computed as:

\begin{equation} \mathbf{Y}_{i,j,f} = \sum_{k_1=0}^{K-1} \sum_{k_2=0}^{K-1} \sum_{d=0}^{D-1} \mathbf{W}_{k_1,k_2,d,f} \cdot \mathbf{X}_{i+k_1,j+k_2,d} + b_f \end{equation}

where $\mathbf{Y} \in \mathbb{R}^{H' \times W' \times F}$ is the output tensor, and $b_f$ is a learnable bias term for each filter.

\subsubsection{Pooling Layers}

Pooling layers are used to downsample the spatial dimensions of the feature maps, reducing the computational complexity and providing translation invariance. The most common pooling operations are max pooling and average pooling. Given an input tensor $\mathbf{X} \in \mathbb{R}^{H \times W \times D}$ and a pooling window of size $P \times P$, the output of a max pooling layer is:

\begin{equation} \mathbf{Y}_{i,j,d} = \max_{p_1} \max_{p_2} \mathbf{X}_{iP+p_1,jP+p_2,d} \end{equation}

\subsubsection{Residual Connections}

Residual connections, introduced in the ResNet architecture \cite{he2016deep}, are a technique to mitigate the vanishing gradient problem and enable the training of deeper neural networks. In a residual block, the input $\mathbf{X}$ is added to the output of one or more convolutional layers $\mathcal{F}(\mathbf{X})$, forming a skip connection:

\begin{equation} \mathbf{Y} = \mathcal{F}(\mathbf{X}) + \mathbf{X} \end{equation}

This allows the gradient to flow directly through the skip connection, facilitating the training of deep networks. If the dimensions of $\mathbf{X}$ and $\mathcal{F}(\mathbf{X})$ do not match, a linear projection $\mathbf{W}_s$ can be used to align the dimensions:

\begin{equation} \mathbf{Y} = \mathcal{F}(\mathbf{X}) + \mathbf{W}_s\mathbf{X} \end{equation}

Residual connections have become a fundamental component of many state-of-the-art CNN architectures, enabling the successful training of networks with hundreds or even thousands of layers.

The three types of operations mentioned above constitute our operations. The specific operations will be detailed in section [ref].

\subsection{DARTS}

DARTS treats neural networks as stacks of $L$ layers of cells. Each cell is composed of $N$ nodes, forming a directed acyclic graph (DAG). Each cell receives the outputs from the two preceding cells as its own inputs, corresponding to the first two nodes (also called input nodes) within the cell. Inside the cell, each node receives the outputs from its preceding nodes and produces an output. Each edge between nodes has a specific transformation operation, denoted as $o\in \mathcal{O}$. Thus, the output of each node can be expressed as:

\begin{equation}
    x^{(j)} = \sum_{i<j} o^{(i,j)}(x^{(i)})
\end{equation}

Unlike ENAS\cite{krizhevsky2009learning}, DARTS does not directly search in the discrete operation space. Instead, it relaxes the problem into a continuous optimization problem. DARTS introduces an architecture parameter $\alpha\in R^{|\mathcal{O}|}$, which is then converted into weights using softmax to assign weights to each operation:

\begin{equation}
    \label{eq:darts-relaxed}
    \overline{o}^{(i,j)}(x)=\sum_{o\in\mathcal{O}}\text{softmax}(\alpha)_{o}(x)
\end{equation}

A neural network can be denoted as $M(w, \alpha)$, where $w$ is the weight parameter. The objective of traditional NAS methods is to minimize the validation loss. As it involves two parameters, $w$ and $\alpha$, the problem is usually formulated as a joint optimization problem: $\min_{w, \alpha} \mathcal{L}(\alpha, w)$. DARTS reformulates it as a bilevel optimization problem by considering $\alpha$ as the upper level variable and $w$ as the lower level variable:

\begin{align}
    \label{eq:darts-bilevel}
    \min_{\alpha} &\, \mathcal{L}_{\text{val}}(w^*(\alpha), \alpha)\\
    \text{s.t.} &\, w^*(\alpha) = \mathrm{argmin}_{w} \mathcal{L}_{\text{train}}(w, \alpha)
\end{align}

However, it is difficult to solve the bilevel optimization problem directly. Computing the accurate gradient of $\mathcal{L}_{val}$ is prohibitively expensive because of inner optimization. Therefore, DARTS proposed a single-step approximation:

\begin{equation}
\label{eq:darts-update}
    \begin{split}
        &\nabla_{\alpha}\mathcal{L}_{\text{val}}(w^*(\alpha), \alpha) \approx\\
        &\nabla_{\alpha}\mathcal{L}_{\text{val}}(w - \xi \nabla_{w}\mathcal{L}_{\text{train}}(w,\alpha), \alpha)
    \end{split}
\end{equation}

Thus, DARTS can iteratively use the above formulas to update on the training set and validation set to solve the bilevel optimization problem. 

\subsection{Bilateral Branch Network}

\begin{figure*}[ht]
    \centering
    \includegraphics[scale=0.75]{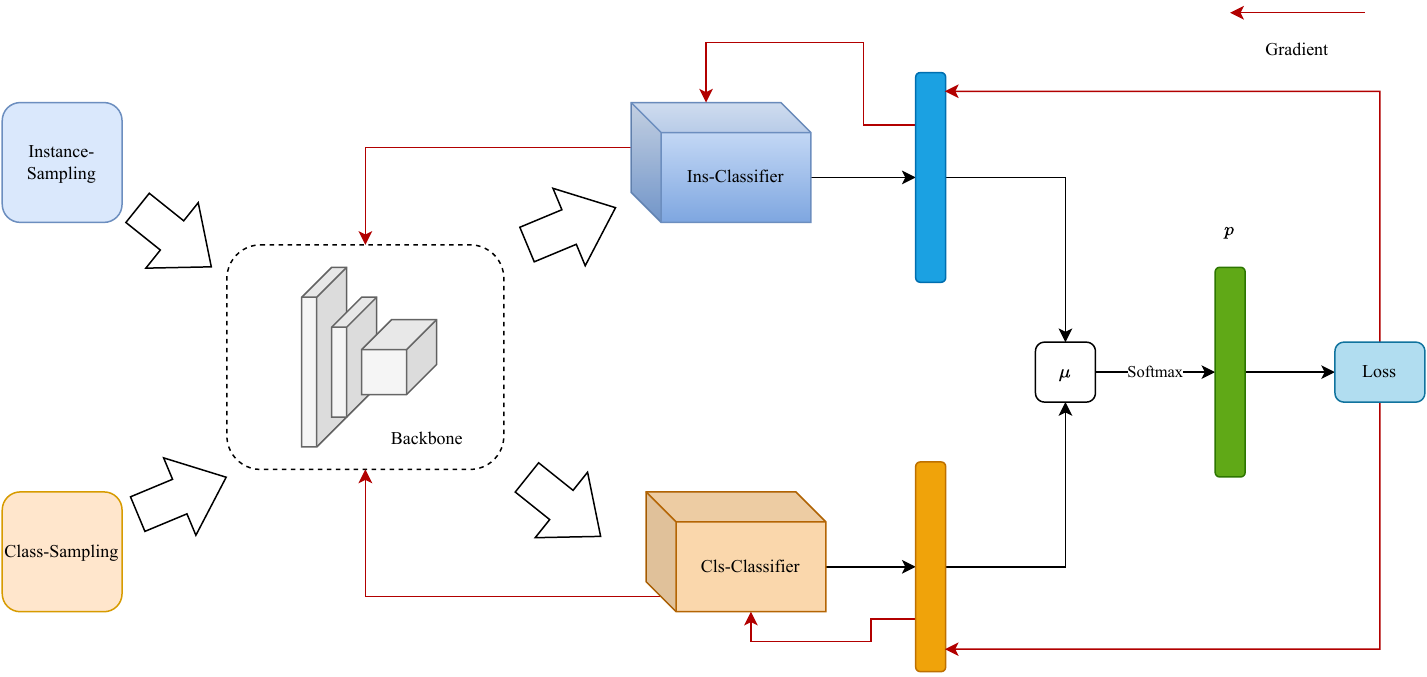}
    \caption{The architecture of BBN. After sampling outputs from two different data sources, they are simultaneously fed into the backbone. The output hidden layer vectors are then fed into two classification heads. Finally, they are mixed according to $\mu$ to obtain the mixed probability vector. The gradient flow is completely a reverse process, resulting in different magnitude of update of each component.}
    \label{fig:bbn}
\end{figure*}

Typically, a classification model is divided into two components: the feature extractor and the classifier. According to \cite{zhou2020bbn}, training directly on the raw data distribution can yield a good feature extractor, but the classifier's performance is usually poor. On the other hand, rebalancing strategies improve classifier performance but result in a less effective feature extractor. The cumulative learning strategy attempts to balance both, enabling the feature extractor to benefit from the raw data while fine-tuning the classifier head with a resampled distribution. \cite{zhou2020bbn} proposed a novel network architecture called BBN.

As shown in Figure \ref{fig:bbn}, a typical cumulative learning strategy consists of a backbone and two classification heads, corresponding to the inputs of instance-sampling and class-sampling, respectively. The logits outputted by these heads are weighted averaged by $\mu$(called \textit{mixing ratio}), and the loss is calculated using the similarly weighted averaged labels. The parameter update formulas for the three components are as follows, where $\xi$ represents the learning rate: 

\begin{equation}
\label{eq:gradient-descend}
    \Theta = \Theta - \xi (\mu\nabla_{\Theta}\mathcal{L}_{ins} + (1-\mu)\nabla_{\Theta}\mathcal{L}_{cls})
\end{equation}

When applying the DARTS algorithm, each component has an architecture parameter in addition to the weight parameters. In the original DARTS algorithm, there are two types of cells: normal cells and reduction cells. Since the classification heads only contain a single residual block, we set the classification heads to include only normal cells. The inputs to the classification heads are the outputs of the last two cells of the backbone. The architecture parameters of each classifier head and backbone are different. 

The update formula can be described as following:
\begin{gather}
    \label{eq:darts-updating-bbn}
    w^*(\alpha) = w-\xi(\mu \nabla_{w}\mathcal{L}_{ \text{ins, train}} + (1-\mu)\nabla_{w}\mathcal{L}_{\text{cls, train}})
\end{gather}
and then apply (\ref{eq:darts-update}).

Though updating rule is same for each component, the input data distribution of each component is quite different.

Concretely, we have an origin dataset with distribution of $\mathcal{D}$, and we use a class-sampling sampler to construct a balanced distribution $\mathcal{D}'$. Let the number of categories be $\mathcal{C}$, and the number of samples for each category be $N_c$. Then, the distribution after sampling can be written as 
\begin{equation}
\label{eq:re-balancing}
    P_{\mathcal{D'}}(x\mid x \text{ in class c}) = \frac{1}{N_c}P_{\mathcal{D}}(x\mid x \text{ in class c})
\end{equation}

For backbone, it receives a mix of $\mathcal{D}$ and $\mathcal{D'}$. For instance-sampling head, it only receives $\mathcal{D}$ and for class-sampling head, it is $\mathcal{D}'$. 

The detailed algorithm is described as Algorithm \ref{alg:bbn-darts}.

\begin{algorithm}[ht]
\caption{BBN with DARTS}
\label{alg:bbn-darts}

\begin{algorithmic}[1]
    \item[] \textbf{Input:} $\mathcal{D} = {(x,y)}$, backbone $f_{bb}(\cdot\mid w_{bb}, \alpha_{bb})$, two classifier heads $f_{ins}(\cdot\mid w_{ins}, \alpha_{ins}), f_{cls}(\cdot\mid w_{cls}, \alpha_{cls})$, maximum epochs $T$.
    
    \STATE Initialize $w, \alpha$ for all components.
    \FOR{t = 1 \textbf{to} T}
    \STATE sample $(x_{ins},y_{ins}) \sim \mathcal{D}$
    \STATE sample $(x_{cls},y_{cls}) \sim \mathcal{D'}$ according to (\ref{eq:re-balancing})
    \STATE $\mathcal{B} = (x_{ins}, x_{cls})$
    \STATE $h = f_{bb}(\mathcal{B})$
    \STATE $o_{ins} = f_{ins}(h)$
    \STATE $o_{cls} = f_{cls}(h)$
    \STATE $\mu = 1-(\frac{t}{T})^2$
    \STATE $\boldsymbol{p} = \mathrm{Softmax}(\mu o_{ins} + (1-\mu) o_{cls})$
    \STATE $\mathcal{L} = \mu\mathbf{CE}(\boldsymbol{p}, y_{ins}) + (1-\mu)\mathbf{CE}(\boldsymbol{p}, y_{cls})$
    \item[] \# \text{\textbf{CE} refers to Cross Entropy}
    \STATE back propagate $\mathcal{L}$ according to (\ref{eq:darts-updating-bbn}) and (\ref{eq:darts-update})
    \ENDFOR
    \RETURN $\Theta_{bb},\Theta_{ins},\Theta_{cls}$
\end{algorithmic}

\end{algorithm}

\subsection{Motivation}

%%展示
%%1. 直接使用re-sampling的性能下降
%%2. 展示训练loss图像.
%%3. 分析. 展示架构变化图像
%%4. 展示mix-ratio变化图.

As an initial experiment, we constructed a long-tailed distribution dataset using the standard image classification dataset CIFAR10, where the class with the fewest samples has only 1/100 of the samples of the class with the most samples. We tested three settings: 1. Using DARTS directly. 2. Using re-sampling for rebalancing. 3. Using BBN+DARTS. The results are shown in the Table \ref{tab:prelim-exp}:

{
\linespread{1.5}
\begin{table}[ht]
\caption{Preliminary Experiment on CIFAR-10}
\label{tab:prelim-exp}

\centering
\begin{tabular}{ll}
\hline
Method            & Accuracy \\ \hline
DARTS             & 0.64     \\
DARTS+Re-sampling & 0.61     \\
DARTS+BBN         & 0.52    \\ \hline
\end{tabular}

\end{table}
}

\begin{figure*}
    \centering
    \includegraphics[scale=0.7]{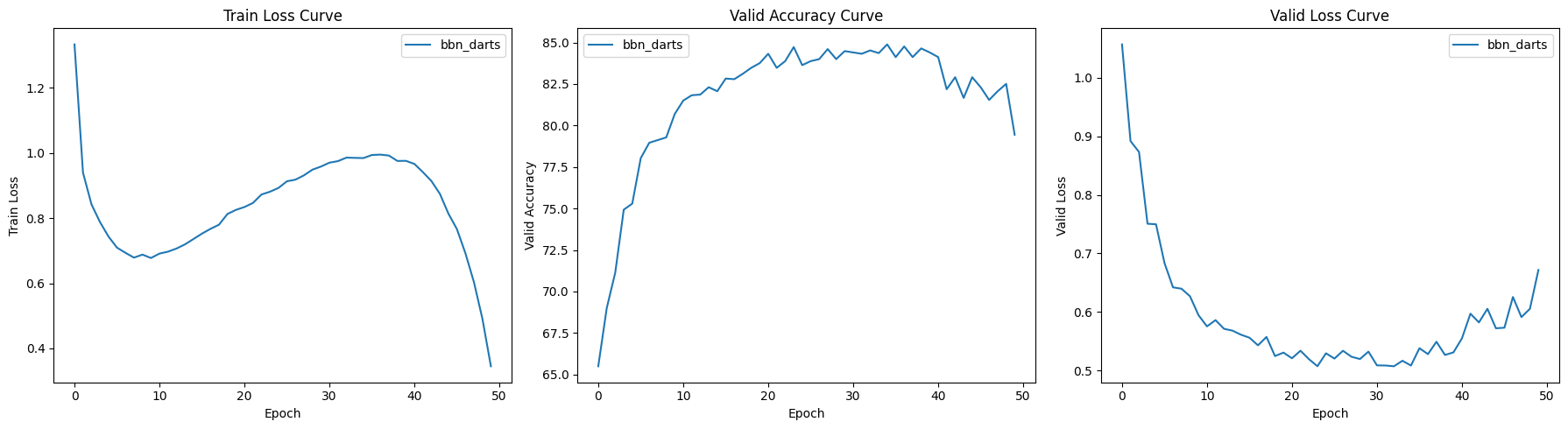}
    \caption{Training curve of simply combining BBN and DARTS. During training, the loss initially decreased, then experienced a strange increase, and subsequently decreased rapidly. However, contrary to this, when the training loss increased, the validation loss remained relatively unchanged; and when the training loss decreased, the validation performance started to decline. Typically, this would suggest overfitting, but judging from the training process, the class-sampling classifier head should actually be underfitting.}
    \label{fig:training-curve-simple}
\end{figure*}

It can be observed that the methods traditionally believed to improve learning outcomes actually lead to a decline in performance. Among them, the performance of BBN+DARTS decreases drastically, far from the approximately 80\% accuracy reported on balanced datasets \cite{zhou2020bbn}.

From the Figure\ref{fig:training-curve-simple}, it can be seen that the training loss exhibits a $\sim$ shape. It is not difficult to deduce from the mixing ratio curve that the training loss has anomalously increased due to the network gradually tilting towards the under-trained Class-sampling head. However, curiously, when the training loss rises, the validation accuracy and loss do not change significantly; instead, it is when the training loss rapidly decreases that the validation accuracy and loss begin to deteriorate. We conjecture that this is due to the difference between the mixing ratio during training and that during testing. Accordingly, we plotted the curve of the final model performance varying with the mixing ratio, and the results were as expected, illustrated in Figure \ref{fig:prelim-acc-mix}. This implies that the belief in BBN, which the two branches are equal, doesn't hold. In fact, the class-sampling head benefits the overall performance more.

\begin{figure*}[ht]
    \centering
    \includegraphics[scale=0.28]{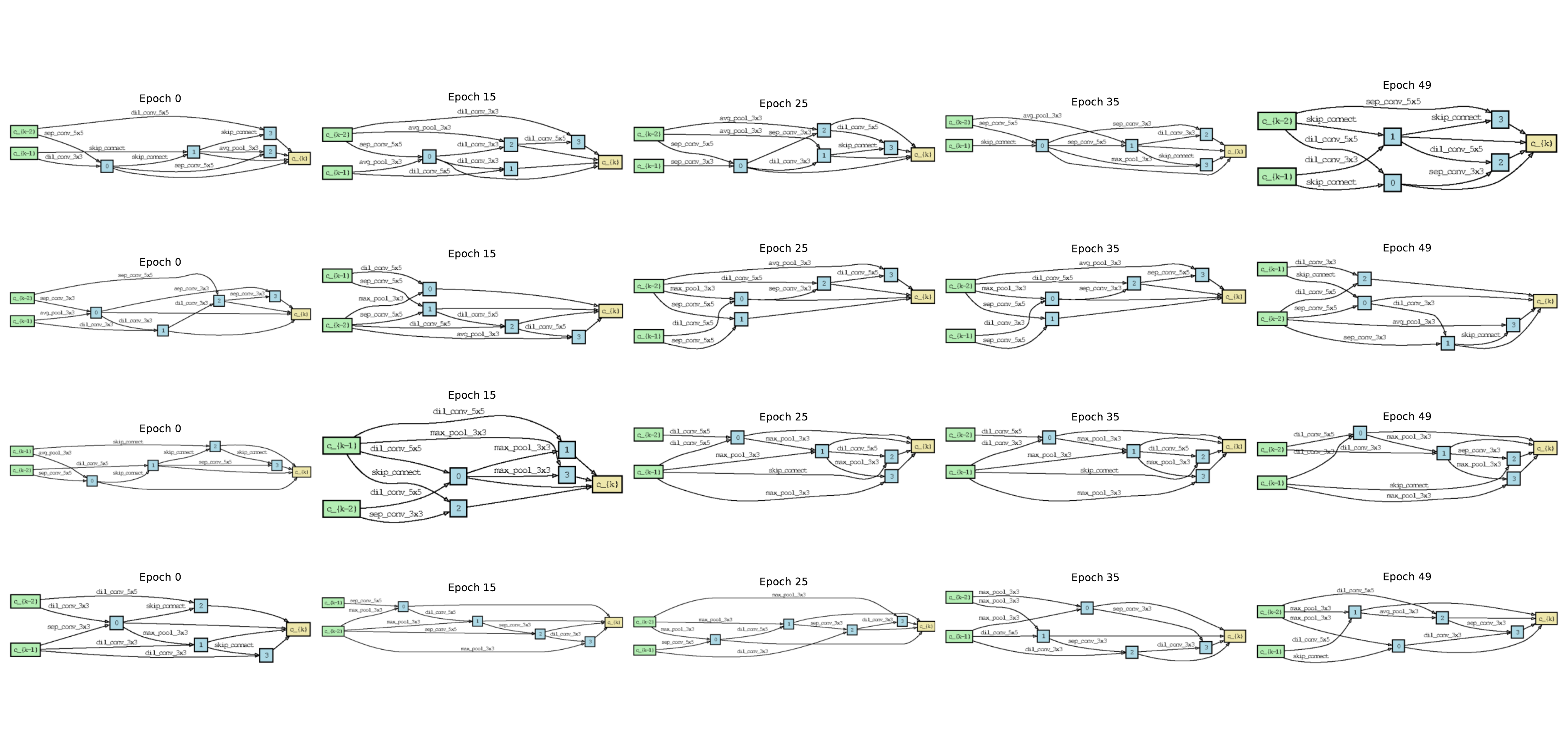}
    \caption{Visualization of the architecture as the epochs change. From top to bottom, they are: Backbone normal cell, Backbone reduction cell, instance-sampling head, and class-sampling head. It can be observed that the backbone does not stably converge even until the end. The instance-sampling head is consistent with mixing ratio, showing no updates in the later stages of training; the class-sampling head, on the other hand, exhibits the opposite behavior.}
    \label{fig:prelim-arch}
\end{figure*}

\begin{figure}[ht]
    \centering
    \includegraphics[scale=0.4]{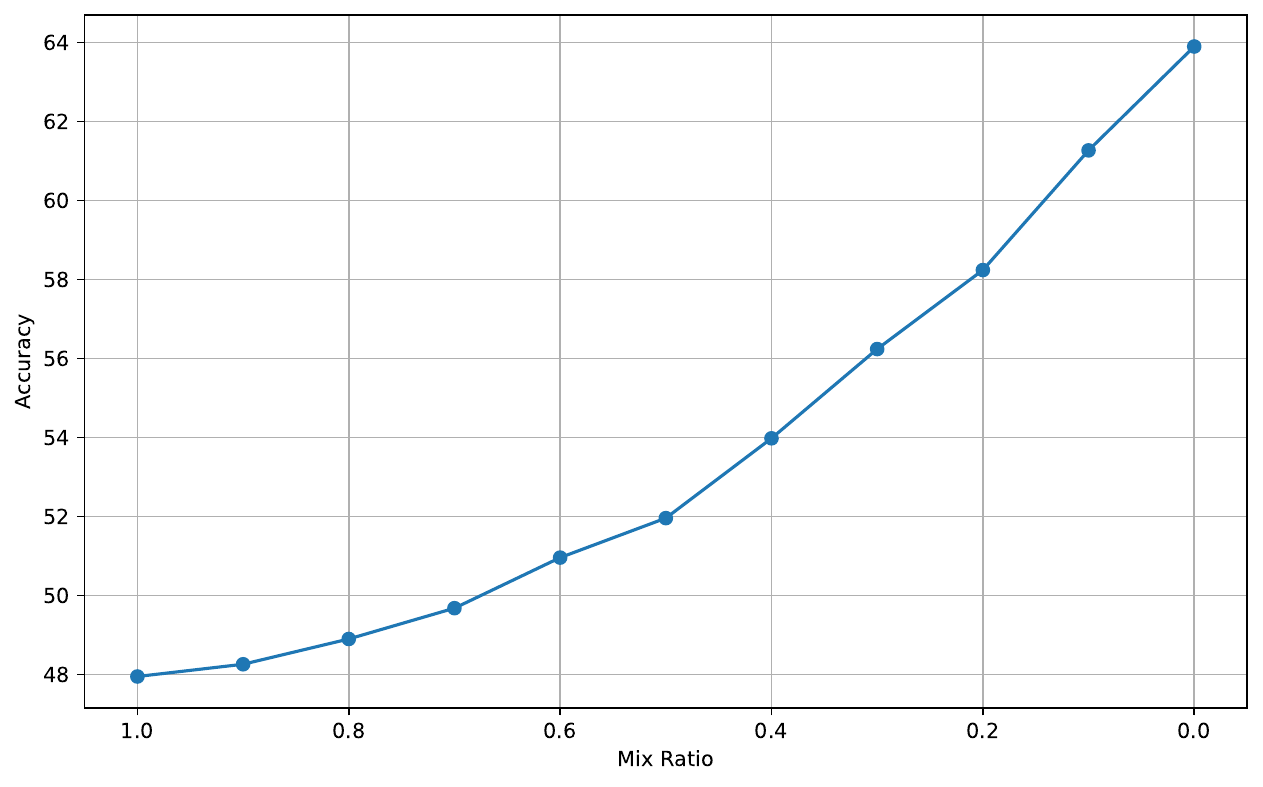}
    \caption{Accuracy versus mixing ratio}
    \label{fig:prelim-acc-mix}
\end{figure}

Another issue is that during training, the architecture parameters of the backbone continue to change rapidly in the later stages, which hinders the optimization process of the class-sampling head. As shown in Figure \ref{fig:prelim-arch}, even in the later stages, the backbone's architecture is still constantly changing, preventing the class-sampling head from being fully updated.

\section{Method}

In this section, we present our proposed approach to address the challenges of neural architecture search on long-tailed datasets. 

\subsection{Cells}

In DARTS, two types of cells are employed: normal cells and reduction cells. Normal cells use a stride of 1 to maintain the resolution, while reduction cells utilize a stride of 2 for downsampling. Each Cell consists of two fixed input nodes and one output node, so there are only $N-3$ actual internal nodes, where $N$ is the total number of nodes in the Cell.  The operation primitives for both cell types belong to the set $\mathcal{O}$.  In the BBN, there are three components in total. As each component serves a different purpose, we assign a separate set of structural parameters to each of them. For the backbone, similar to DARTS, we adopt a scheme of stacking cells; however, for the two classification heads, we only use a single normal cell. Thus, our architecture parameters are encoded as $(\alpha_{bb}, \alpha_{ins}, \alpha_{cls})$. By adopting this approach, we eliminate the need for additional architectural parameters for extra reduction cells, thus simplifying the optimization process.

\subsection{Heterogeneous LR scheduling}

%%too rapid arch change

This subsection provides a detailed description of our proposed heterogeneous learning rate (LR) scheduling strategy. 

Overall, this strategy aims to facilitate the representation learning of the backbone in the early stages of training while suppressing the interference of the backbone's architectural parameters on the fine-tuning process of the class-sampling classification heads in the later stages, striking a balance between the two. We can easily prove that, throughout the entire training process, the strength of the gradient update signal received by the backbone remains constant, while the gradient update signal for the class-sampling classification heads gradually increases as the mixing ratio shifts.

\begin{theorem}
\label{theo:gradient-signal}
During the training process, the magnitude of the gradients for the backbone remains constant, while the gradients for the classification heads vary with the mixing ratio.
\end{theorem}

\begin{proof}

Denote loss as $\mathcal{L}(\alpha_{bb}, \alpha_{ins}, \alpha_{cls})$. For classification head, we have 
\begin{equation}
    \mathcal{L}(\alpha_{ins}, \alpha_{cls}) = \mu \mathcal{L}(\alpha_{ins}) + (1-\mu)\mathcal{L}(\alpha_{cls})
\end{equation}

Hence, the gradient is affected by mixing ratio $\mu$:
\begin{gather}
    \nabla_{\alpha_{ins}}\mathcal{L} = \mu\nabla_{\alpha_{ins}}\mathcal{L}(\alpha_{ins}) \\
    \nabla_{\alpha_{cls}}\mathcal{L} = (1-\mu)\nabla_{\alpha_{cls}}\mathcal{L}(\alpha_{cls}) 
\end{gather}

However, for backbone parameter, because $\alpha_{ins}$ and $\alpha_{cls}$ are irrelevant, we have:
\begin{align}
 \nonumber
    \nabla_{\alpha_{bb}}\mathcal{L} &= \nabla_{\alpha_{bb}}(\mu \mathcal{L}(\alpha_{bb},\alpha_{ins}) + (1-\mu)\mathcal{L}(\alpha_{bb},\alpha_{cls}))\\ \nonumber
    & = \mu \nabla_{\alpha_{bb}}\mathcal{L}(\alpha_{bb}) + (1-\mu)\nabla_{\alpha_{bb}}\mathcal{L}(\alpha_{bb})\\
    & = \nabla_{\alpha_{bb}}\mathcal{L}(\alpha_{bb})
\end{align}

\textbf{Remark}: This formula appears to merely rearrange symbols, but it has significant practical implications. In the context of the architecture (Figure \ref{fig:bbn}), the gradients fed back from the two classification heads converge again in the backbone due to weight sharing, resulting in a consistent gradient magnitude received by the backbone during training.

\end{proof}

This contradicts the usual understanding that the backbone aids in training a representation. Even in the later stages of training, when the gradients are primarily dominated by the class-sampling head, the backbone still receives updates similar to those at the initial stages of training. This results in destructive updates to the weights, disrupting the learned representations, and is also detrimental to the training of the class-sampling head, as part of the updates is transferred to the backbone.

We noticed that the mixing ratio effectively serves as a learning rate adjustment to some extent, but only for the two classification heads. Therefore, we apply a negative exponential signal to the learning rate of the backbone, causing the learning rate of its architecture parameters to decrease as the epochs progress. Specifically, at epoch i, we transform the learning rate as follows:
\begin{equation}
    \label{eq:lr-scale}
    \xi_{i} = \xi_{0}\cdot (1-e^{-\tau \mu})
\end{equation}
where $\mu$ is the mixing ratio at epoch $i$. 

\begin{figure}[ht]
    \centering
    \includegraphics[scale=0.4]{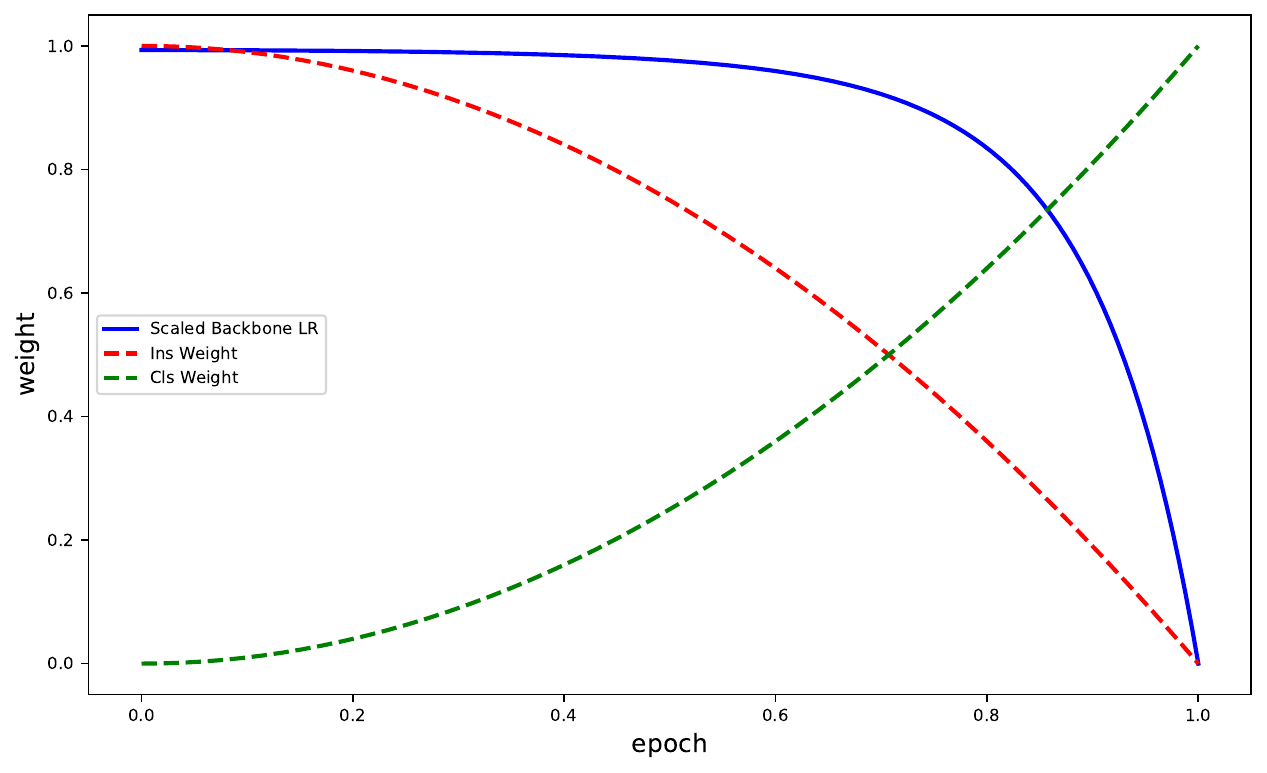}
    \caption{Weight versus normalized epoch. Ins Weight is identical to the mixing ratio $\mu$. }
    \label{fig:decay}
\end{figure}

The image of this transformation factor is shown in the Figure \ref{fig:decay}. It can be seen that for a while initially, there is almost no impact on the learning rate. However, as the mixing ratio gradually approaches 0, the scaling factor quickly drops to 0. The scaling factor decreases faster than the mixing ratio approaches 0, thus enabling the "soft freezing" of the backbone's parameters ahead of time.

\subsection{Mixing ratio}

Based on previous observations, in the context of combining with the DARTS algorithm, the two branches are not equivalent in the final inference stage. Experimental results show that the instance-sampling branch alone can achieve a moderate level of performance, while the class-sampling branch can achieve higher performance, which is consistent with common understanding. In the original design of BBN, the mixing ratio is adjusted using a quadratic descent strategy. This function is not strictly symmetric about $x=0.5$, leading to a slightly longer emphasis on the instance-sampling branch, resulting in insufficient training of the class-sampling branch. This coincides the strange training curve shown in Figure \ref{fig:training-curve-simple}. We attempted to use a new function, the reverse sigmoid function, to regulate the mixing ratio. Our aim is to train the network fairly on both branches.

\begin{equation}
    RS(x) = 1-\text{Sigmoid}(\frac{t-T/2}{T/2}*k)
\end{equation}

where $k$ is an adjustment factor. The function is (centrally) symmetric with respect to 0.5, initially spending a prolonged period "soaking" in the instance-sampling head, and then rapidly switching to the class-sampling head.

\section{Experiment}
\subsection{Settings}

\textbf{Dataset} CIFAR10\cite{krizhevsky2009learning} is a widely used dataset in the field of machine learning and computer vision. It consists of 60,000 32x32 color images divided into 10 different classes, such as airplanes, cars, birds, cats, etc. Each class contains 6,000 images. The dataset is split into 50,000 training images and 10,000 testing images. We employed a commonly used method for constructing imbalanced datasets, which involves sampling from each class at a proportionally exponential rate. The degree of imbalance is illustrated by the ratio of the largest class to the smallest class, and in our case, the imbalance ratio is 100. The training set was further divided into a training set and a validation set at a ratio of 80/20.

\textbf{Implementation} The experiment was conducted using PyTorch\footnote{https://pytorch.org/}, with partial reuse of the DARTS code\footnote{https://github.com/quark0/darts}. It was executed on a single A6000 GPU, with each training run taking approximately 8 hours, and peak memory usage reaching up to 35GB. We applied 4-bit padding to the input and performed normalization. To accomplish the heterogeneous LR scheduling for various components, we completely rewrote the architecture section of DARTS and introduced a new component called ArchOptimizer, which is responsible for controlling the details of gradient backpropagation in the DARTS algorithm.

\textbf{Hyperparameters} We trained using a small 8-layer backbone network and a network composed of two single cells. The batch size was set to 128, which effectively became 256 due to the simultaneous sampling from two data sources. The weight learning rate was set to 0.02 and used cosine annealing; the architecture learning rate also started at 0.02 and was scheduled according to the strategy previously described. We employed a weight decay of 3e-4 to regulate parameters. We set maximum epoch $T=50$.

Due to time and cost constraints, we only trained and validated on a small 8-layer network, and did not follow the common practice of fully training on a larger network (12 layers or more) with the cell structure obtained for over 200 epochs. This resulted in a general decrease in accuracy, from around 75\% to around 60\%. In the future, when time and cost permit, we will consider conducting a thorough and complete experiment, at which point we can compare our results with other methods.

\subsection{Main Results}

The overall accuracy metrics are as follows(HLS refers to Heterogeneous LR Scheduling):

{
\linespread{1.5}
\begin{table}[!ht]
\caption{Accuracy of Methods on Long-Tailed CIFAR-10}
\label{tab:main-results}
\centering
\begin{tabular}{ll}
\hline
Method            & Accuracy \\
\hline
DARTS             & 64.56    \\
DARTS+Re-sampling & 61.20    \\
DARTS+BBN         & 52.14    \\
HLS              & \textbf{65.12}  \\
HLS + Reverse Sigmoid & 61.85 \\ 
HLS + Continous Learning & 63.12 \\ \hline
\end{tabular}
\end{table}
}

As seen in Table \ref{tab:main-results}, the best results were achieved using only the HLS method. Employing a reverse sigmoid mixing ratio or further fine-tuning \cite{cao2019learning} actually led to a decrease in performance. All these results were reported based on the best mixing ratio, and we will provide more detailed information in the subsequent sections.

\subsection{Ablation: Backbone LR scaling}

\begin{figure}
    \centering    \includegraphics[scale=0.35]{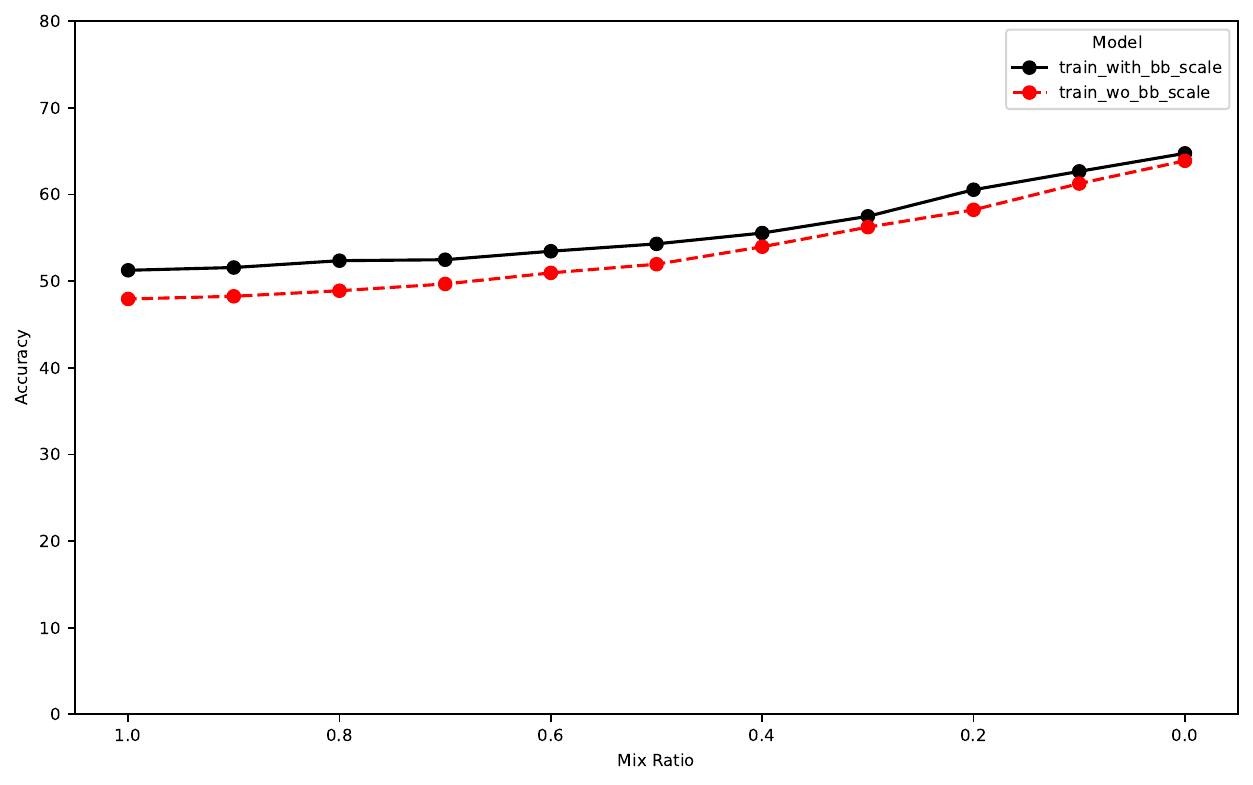}
    \caption{Accuracy versus Mixing Ratio wether using scaled backbone 
 lr or not. Scaled LR obtains a significant improvement.}
    \label{fig:scaled-lr-acc}
\end{figure}

As observed in Figure \ref{fig:scaled-lr-acc}, the method using scaled backbone lr outperformed the method without scaled backbone lr across all mixing ratios. This indicates that when training with BBN+DARTS, it is necessary to gradually reduce the update magnitude of the Backbone, which aligns with our analysis.

\begin{figure*}[t]
    \centering
    \includegraphics[scale=0.28]{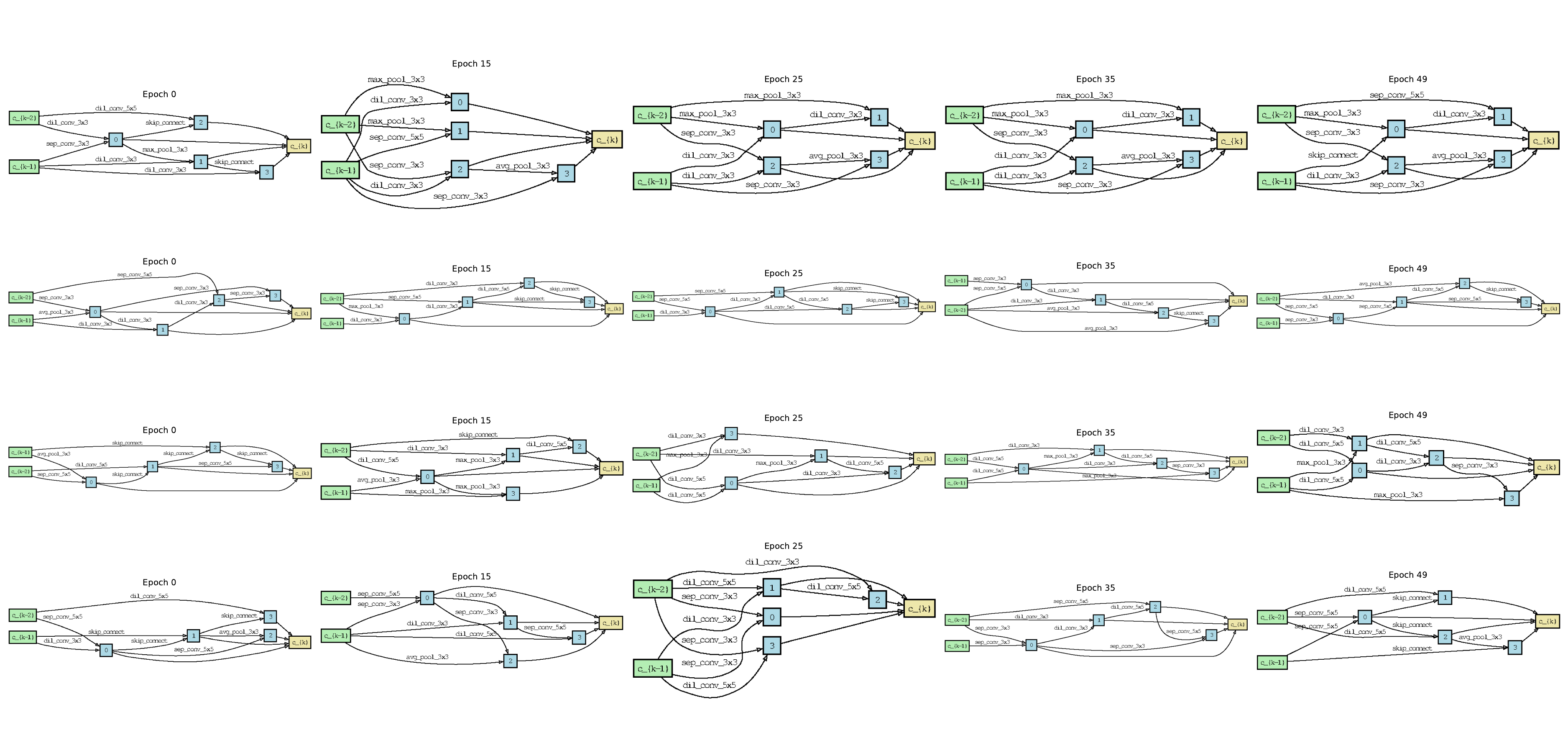}
    \caption{Visualization of the architecture as the epochs change. From top to bottom, they are: Backbone normal cell, Backbone reduction cell, instance-sampling head, and class-sampling head.}
    \label{fig:exp-arch}
\end{figure*}

To strengthen our assertion, we also examined the architectures generated during training, as shown in Figure \ref{fig:exp-arch}. Comparing it with Figure \ref{fig:prelim-arch}, we can see that the structure of the backbone using scaled lr scheduling stabilized in the later stages of training, avoiding interference with the training of the class-sampling head. This is consistent with our expectations.

\subsection{Continous Training}

\begin{figure}[h]
    \centering
    \includegraphics[scale=0.35]{ 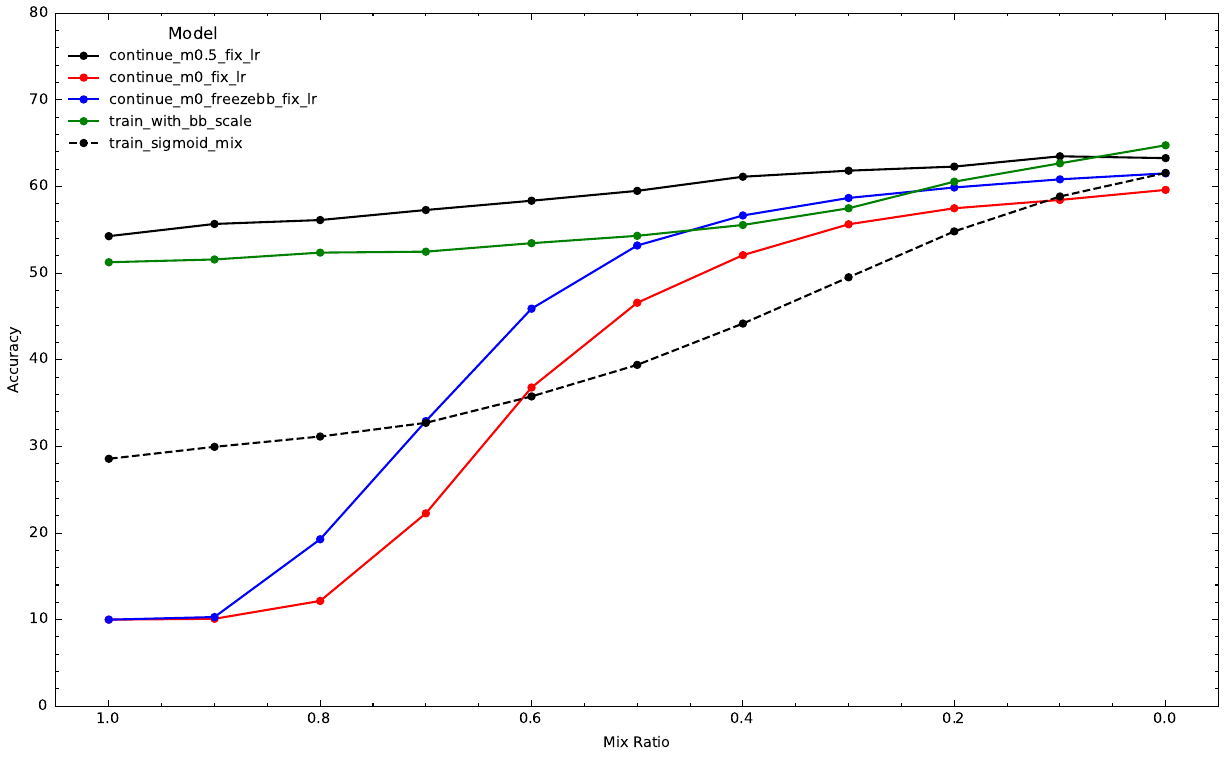}
    \caption{Accuracy versus Mixing Ratio with different training strategy}
    \label{fig:con}
\end{figure}

Given that the performance is consistently best when the mixing ratio is set to 0 during testing, i.e., using only the class-sampling head, it naturally raises the question: Could we benefit from allowing the class-sampling head to train for a longer period, or by training with a lower mixing ratio for an extended duration? To explore this, we conducted tests. We evaluated the following scenarios: 1. Training with $\mu=0$, only for class-sampling branch. 2. Training with a mixing ratio of 0.5. 3. Employing a reverse sigmoid during the training process. 4. Training with the backbone frozen and solely focusing on the class-sampling head. The results are as shown in \ref{fig:con}.

We observed that training with a mixing ratio of 0.5 yielded the most stable performance; however, the peak metrics were slightly lower than those achieved without extended training. The other extended training methods suffered from catastrophic forgetting, where the instance-sampling branch became completely ineffective, indicating that the backbone parameters were severely compromised (as the instance-sampling classifier was not updated).

\section{Conclusion}

In this paper, we studied the problems which DARTS encounters in imbalanced
dataset. Specifically, we tried to utilize the bilateral-branch
network(BBN) to mitigate the performance degrade.

Applying DARTS directly to cifar-10 with imbalance ratio 100
results in approximately 65\% accuracy.
While using re-sampling strategy leads to a significant performance
degradation of 60\%.
The involvement of naive BBN makes a worse performance, and
with our heterogeneous learning rate scheduling and mix ratio
strategy, we achieve an approximately 65\% accuracy too.

This suggests that the re-sampling strategy would inherently affect
the performance of the DARTS algorithm, and we hope that this
work can provide new insights for future researchers.

\bibliographystyle{IEEEtran}
\bibliography{references}

\end{document}